\RequirePackage[T1]{fontenc}
\documentclass{IEEEtran}
\usepackage{cite}
\usepackage{amsmath,amssymb,amsfonts}
\usepackage{algorithmic}
\usepackage{graphicx}
\usepackage{algorithm}
\usepackage[caption=false,font=footnotesize]{subfig}
\usepackage{xcolor}
\usepackage[hidelinks]{hyperref}
\usepackage{placeins}

\newif\ifrev
\revfalse
\definecolor{revisionround}{RGB}{150,0,105}
\ifrev
  \newcommand{\rev}[1]{\textcolor{blue}{#1}}
  \newcommand{\revnew}[1]{\leavevmode{\color{blue}#1}}
\else
  \newcommand{\rev}[1]{#1}
  \newcommand{\revnew}[1]{#1}
\fi

\newtheorem{theorem}{Theorem}
\newtheorem{lemma}{Lemma}
\newtheorem{prop}{Proposition}
\newtheorem{definition}{Definition}
\newtheorem{assumption}{Assumption}
\newtheorem{remark}{Remark}
\newtheorem{example}{Example}
\newtheorem{corollary}{Corollary}

\newcommand{\Comp}{\mathsf{Comp}}

\begin{document}

\title{Formalizing and Estimating Task-Space Complexity for Zero-shot Generalization}

\author{Jung-Hoon Cho, Heling Zhang, Siqi Du, Roy Dong, and Cathy Wu
\thanks{\rev{A preliminary version of this work appeared in~\cite{cho2024formalizing}.}}
\thanks{J.-H. Cho is with the Department of Civil and Environmental Engineering and the Laboratory for Information \& Decision Systems, Massachusetts Institute of Technology, Cambridge, MA 02139, USA. (e-mail: \texttt{jhooncho@mit.edu})}
\thanks{H. Zhang is with the Department of Electrical and Computer Engineering, University of Illinois at Urbana-Champaign, Urbana, IL 61801, USA. (e-mail: \texttt{hzhng120@illinois.edu})}
\thanks{S. Du and R. Dong are with the Department of Industrial and Enterprise Systems Engineering, University of Illinois at Urbana-Champaign, Urbana, IL 61801, USA. (e-mail: \texttt{\{siqidu3,roydong\}@illinois.edu})}
\thanks{C. Wu is with the Laboratory for Information \& Decision Systems; the Institute for Data, Systems, and Society; and the Department of Civil and Environmental Engineering, Massachusetts Institute of Technology, Cambridge, MA 02139, USA. (e-mail: \texttt{cathywu@mit.edu})}}

\maketitle

\begin{abstract}
\rev{\revnew{Generalization is crucial for modern control, but learning across tasks can be costly. Recent work shows that zero-shot transfer, in which a controller designed for a source task is directly deployed to a target task, can achieve strong performance. We formalize task-space complexity, which characterizes the difficulty of achieving sufficiently good performance across a collection of tasks, by casting source selection as a covering problem. We provide bounds on the complexity and demonstrate its applicability.}}
\end{abstract}


\section{Introduction}\label{sec:introduction}
Modern control systems increasingly operate across diverse conditions\rev{:} robots handle objects of varying mass, autonomous vehicles drive on surfaces with different friction, and physical systems \revnew{vary in} stiffness. \rev{\revnew{Following the reinforcement learning (RL) literature~\cite{hallak2015contextual}, we describe these variations by a parameter $\theta$, the \emph{context}, which takes values in a set $\Theta\subset\mathbb R^d$, the \emph{context space}. A context fixes the dynamics of the system and the control objective. We call the control problem at a context a \emph{task} and the family of tasks the \emph{task space}.}}

\rev{\revnew{The generalization problem is to perform well across the task space. One option is a single controller for all tasks, usually at the expense of performance.} Robust control can reduce nominal performance~\cite{zhou1998essentials}, adaptive control can exhibit poor transient performance during adaptation~\cite{ioannou1996robust,narendra1994improving}, and a single RL policy trained on all tasks can underperform task-specific policies~\cite{parisotto2016actormimic,cho2024model}. \revnew{Another option is a separate controller for every task, which requires repeated design or training. A library of local controllers, each designed for one task and deployed on nearby tasks, is an intermediate option.}}

\rev{\revnew{We consider \emph{zero-shot transfer}, which deploys a controller designed for one task on another without redesign or online adaptation. It is widely studied in RL~\cite{kirk2023survey}, and our earlier studies show that it is effective for learned controllers.} In traffic control with advisories held for 0.1 to 40~s, controllers trained at a few hold durations transferred zero-shot to the others, whereas deep RL applied directly did not generalize~\cite{cho2025temporal}. Model-based transfer learning selects source tasks with a loss model linear in context distance and reports up to $43\times$ better sample efficiency than per-task or multi-task training~\cite{cho2024model}.}

\rev{\revnew{We define the \emph{source} as the task a controller is designed for and the \emph{target} as a task on which it is deployed; the source context plays a role similar to a design point in gain scheduling~\cite{rugh2000research}. Formally, a design or training procedure maps a source context $\theta$ to a controller $\pi_\theta$, a feedback law acting on the state and possibly the context. Its performance at a target context $\tilde\theta$ is $J(\tilde\theta,\pi_\theta)$. We measure transfer by the \emph{source-relative performance loss} $J(\theta,\pi_\theta)-J(\tilde\theta,\pi_\theta)$, the performance the controller loses when deployed at the target rather than at its source. This loss needs only evaluations of the deployed controller, whereas suboptimality at the target would need the optimal performance at every target, which zero-shot transfer avoids computing. Context distance is an unreliable predictor of this loss, since}} performance can be highly sensitive to certain \revnew{context} changes and insensitive to others\rev{, and transfer can succeed in one direction but fail in the reverse \revnew{direction}.}

\rev{\revnew{In practice, this procedure can be expensive and random. In RL, the trained controller depends on the random seed and checkpoint, and in control,} optimal and robust designs can be fragile to small \revnew{deviations in their implemented} coefficients~\cite{keel1997robust}. \revnew{We therefore collect the controllers that the procedure may produce for $\theta$ in a set $\Pi(\theta)$, the \emph{source ensemble}, and bound the loss uniformly over it. This worst-case loss is the \emph{signed divergence} $D(\theta;\tilde\theta)=\sup_{\pi\in\Pi(\theta)}\bigl(J(\theta,\pi)-J(\tilde\theta,\pi)\bigr)$. Whichever controller $\pi_\theta\in\Pi(\theta)$ the procedure returns, its loss is at most $D(\theta;\tilde\theta)$.}}

\rev{\revnew{Given a tolerance $\varepsilon>0$ on this loss, the targets with $D(\theta;\tilde\theta)\le\varepsilon$ form the \emph{tolerance set} of $\theta$; the source \emph{covers} them, and its controllers \emph{generalize} to them. We ask} \emph{how many source tasks suffice to cover the task space under a prescribed tolerance on performance loss}\revnew{. The answer is the \emph{task-space complexity}, the minimum number of sources whose tolerance sets cover $\Theta$; choosing them is a set-cover problem.}}

Our key contributions are as follows:
\begin{enumerate}
    \item \rev{\textbf{Performance-based source coverage.} We define task-space complexity through the signed divergence, so that any one controller per selected source \revnew{has loss at most $\varepsilon$ on the tasks its source covers}.}
    \item \rev{\textbf{\revnew{Performance versus model smoothness}.} \revnew{Model smoothness is sufficient but not necessary for a small divergence, so performance-based bounds can be far less conservative than model-based ones.}}
    \item \rev{\textbf{Geometric bounds.} We derive \revnew{inner and outer balls for the tolerance sets} and \revnew{bounds on the complexity} from local rates, and \revnew{show that positive spanning is necessary and sufficient for the divergence to grow} at least linearly near \revnew{a uniformly differentiable interior source}.}
    \item \rev{\textbf{\revnew{Greedy source selection}.} We combine rate estimation with greedy set cover and \revnew{distinguish certified continuum coverage from empirical grid coverage. We demonstrate applicability on} mass--spring--damper \revnew{and CartPole systems}.}
\end{enumerate}

\vspace{-0.4em}
\section{Related Work}\label{sec:related-work}
\textbf{Generalization in Control and RL.}
Gain scheduling is a classical approach where controllers are designed for a set of operating points (contexts) and interpolated for intermediate points~\cite{rugh2000research}. \rev{Gridding-based linear parameter-varying (LPV) synthesis imposes its conditions only at finitely many parameter values~\cite{hoffmann2015survey}. Multiple-model adaptive and supervisory control switch among finitely many controllers~\cite{narendra1997adaptive,morse1996supervisory}. Under mild assumptions, finitely many controllers give every plant in an uncertainty set satisfactory performance~\cite{anderson2000multiple}. The number of models can be derived from a robust-performance requirement~\cite{fekri2006issues}, and the placement of models can be optimized by coverage control in the Vinnicombe metric~\cite{kersting2017systematically,vinnicombe1993frequency}. These approaches rely on plant models. Ours uses performance data and holds for every controller in an ensemble, so it also applies to learned ones.}

In RL, domain randomization~\cite{tobin2017domain}, meta-learning~\cite{finn2017model}, and multi-task learning~\cite{caruana1997multitask} \revnew{train} policies that are robust to variations. \rev{\revnew{In control, data-based} transfer stabilization uses data from a source system to stabilize the target~\cite{li2024data}. \revnew{We instead keep the design or training procedure fixed and quantify how many sources it needs.}}

\textbf{Task and Policy Similarity.}
\rev{Cover-Rmax covers the contexts of a contextual Markov decision process (CMDP)~\cite{hallak2015contextual} with balls whose radius} is fixed by known global Lipschitz constants of the model~\cite{modi2018markov}. \rev{The gap and $\nu$-gap metrics measure the distance between plants in terms of robust stabilization~\cite{vinnicombe1993frequency}, and the gap metric can be computed from data~\cite{koenings2018data}. Bisimulation metrics~\cite{ferns2004metrics} and behavioral similarity embeddings~\cite{agarwal2021contrastive} compare states through rewards and transitions and through policy behavior, respectively. \rev{We compare tasks only through the performance loss of given controllers, so the comparison reflects how those controllers transfer, not an intrinsic similarity of the tasks.}}

\rev{\textbf{Coverings of \revnew{Control Tasks}.}}
\rev{The closest work to ours defines the $\alpha$-suboptimal covering number, the minimum number of controllers such that every task in a continuous family has a controller whose cost is within a factor $\alpha$ of the optimal cost, and analyzes it for LQR families~\cite{preiss2021suboptimal}. We bound the additive loss relative to each controller's source performance, uniformly over a source ensemble, which can require more sources than selecting controllers individually.}
\vspace{-0.3em}
\section{Preliminaries}\label{sec:prelim}
\rev{The performance $J(\theta,\pi)$ of controller $\pi$ (a policy in RL) on the task at context $\theta$ is an expected cumulative reward or, in control terms, \revnew{the negative of a cost}, so larger values are better.} For example, for a continuous-time linear system \rev{$\dot x=A(\theta)x+B(\theta)u$} \revnew{with $u=\pi(x)$ and a fixed initial state $x(0)=x_0$}, the \revnew{performance is the negative of the finite-horizon} LQR cost, $J(\theta,\pi) = -\!\int_0^{\rev{T}} \bigl(x^\top Q x + u^\top R u\bigr)\,dt$.

\rev{The source ensemble $\Pi(\theta)$ is the set of controllers that a fixed design or training procedure produces for source $\theta$} (e.g., multiple random seeds, early-stopping checkpoints, perturbation, or nearby hyperparameter settings)\revnew{; it may be infinite, such as all gains within a given distance of a nominal gain}. \rev{A controller may act on the state alone, $\pi(x)$, or on the state and the context, $\pi(x,\theta)$. Its parameters stay fixed during transfer. A context-aware controller evaluated at target $\tilde\theta$ receives $\tilde\theta$ as input, so zero-shot transfer may change its input but not its parameters. We evaluate each controller on a fixed target task and do not consider online switching between controllers.}
\rev{Fix a radius $r>0$ and let $B_r(\theta)=\{\tilde\theta\in\mathbb R^d:\|\tilde\theta-\theta\|_2\le r\}$ denote the closed Euclidean ball. The following assumption holds throughout the \revnew{paper}.}

\begin{assumption}[\rev{Context and performance regularity}]\label{assume:standing}
\rev{(i) $\Theta\subset\mathbb{R}^d$ is nonempty and compact, with no isolated points; (ii) for every $\theta\in\Theta$, $\Pi(\theta)$ is \revnew{nonempty}; (iii) for every $\theta\in\Theta$, \revnew{the functions $J(\cdot,\pi)$, $\pi\in\Pi(\theta)$, are bounded on $\Theta$ uniformly in $\pi$, and each is} \revnew{$\ell_\pi(\theta)$-Lipschitz} on $B_r(\theta)\cap\Theta$\revnew{, where $\bar\ell(\theta):=\sup_{\pi\in\Pi(\theta)}\ell_\pi(\theta)<\infty$}.}
\end{assumption}

\rev{We assume exact evaluations of $J$. The end of Section~\ref{sec:complexity} treats the error from estimating $J$ with finitely many rollouts.}

\rev{\revnew{Let $J^*(\theta):=\sup_{\pi\in\Pi(\theta)}J(\theta,\pi)$ be the best source performance of the ensemble, and let $\pi^*(\theta)$ denote a best member, with $J(\theta,\pi^*(\theta))=J^*(\theta)$, whenever one exists (always for a finite ensemble). The} \emph{source-relative performance loss} \revnew{of $\pi\in\Pi(\theta)$} at a target $\tilde\theta\neq\theta$ is}
\begin{equation}
    \revnew{\Delta J_\pi(\theta, \tilde{\theta}) = J(\theta, \pi) - J(\tilde{\theta}, \pi)}.
    \label{eq:gap}
\end{equation}
\rev{This loss is called the generalization gap in~\cite{kirk2023survey,cho2024model,cho2025temporal}, and it is negative when transfer improves performance. It differs from the target-relative gap \revnew{$J^*(\tilde\theta)-J(\tilde\theta,\pi)$}, which compares the transferred controller with the best \revnew{performance} of the target ensemble and can also be negative when the two ensembles differ. A small source-relative loss preserves source performance but does not certify absolute performance at the target. Unless stated otherwise, every tolerance in this \revnew{paper} bounds the source-relative loss.}
\vspace{-0.2em}
\section{Task Dissimilarity}\label{sec:dissimilarity}
\subsection{Signed Divergence as Dissimilarity}
\begin{definition}[Signed \revnew{divergence}]\label{def:divergence}
    Given \revnew{the source ensemble} $\Pi(\theta_1)$, we define the signed divergence from $\theta_1$ to $\theta_2$ as:
    \begin{equation}
        D(\theta_1; \theta_2) := \sup_{\pi \in \Pi(\theta_1)} \left( J(\theta_1, \pi) - J(\theta_2, \pi) \right).
        \label{eq:divergence}
    \end{equation}
\end{definition}
The signed divergence is monotone in the \revnew{ensemble}\rev{: if} $\Pi_1(\theta_1)\subseteq\Pi_2(\theta_1)$\rev{,} then $D_{\Pi_1}(\theta_1;\theta_2)\le D_{\Pi_2}(\theta_1;\theta_2)$. \rev{\revnew{For a singleton $\Pi(\theta)=\{\pi\}$, the divergence reduces} to the source-relative loss \revnew{of $\pi$ in} \eqref{eq:gap}.} If $J(\cdot,\pi)$ is \revnew{Lipschitz on $B_r(\theta_1)\cap\Theta$ with a common constant} for all $\pi\in\Pi(\theta_1)$, then \revnew{$D(\theta_1;\cdot)$ is Lipschitz there} \rev{(Lemma~\ref{lemma:finite}\revnew{(iii)})}. \revnew{The} signed divergence is asymmetric\revnew{:} $D(\theta_1; \theta_2) \neq D(\theta_2; \theta_1)$ in general\revnew{, and it} need not satisfy the triangle inequality.
\revnew{For a finite ensemble,} the supremum in Definition~\ref{def:divergence} \revnew{is} a maximum\rev{\revnew{, and given} the source \revnew{performances $J(\theta_1,\pi)$}, computing $D(\theta_1;\theta_2)$ requires $|\Pi(\theta_1)|$ evaluations of $J$ at the target. \revnew{Algorithms~\ref{alg:main_alg} and~\ref{alg:L_max_approximation} and the experiments use finite ensembles.}}

\begin{remark}[\rev{What the divergence does and does not measure}]\label{rem:scope}
\rev{The divergence measures how the given ensemble $\Pi(\theta_1)$ generalizes, not an intrinsic dissimilarity of the tasks. Two tasks related by a symmetry, such as moving a box from left to right and from right to left, can have a large divergence when the deployed controllers do not exploit the symmetry. A controller that exploits it can reduce the divergence, but adding reflected controllers to an ensemble cannot, since the supremum in \eqref{eq:divergence} still includes the original members. A negative divergence means that every ensemble member improves at $\theta_2$, so improvements are not penalized. Transfer in the reverse direction is measured separately by $D(\theta_2;\theta_1)$.}
\end{remark}

\begin{remark}[\rev{Directionality}]\label{rem:direction}
Consider a one-step CMDP with start state $s_0$ and terminal states $s_A,s_B$. Action $a_A$ (resp. $a_B$) deterministically reaches $s_A$ (resp. $s_B$). Let $\Pi=\{\pi_A,\pi_B\}$ where $\pi_A(s_0)=a_A$ and $\pi_B(s_0)=a_B$. Rewards depend on context: at $\theta_1$, $(r(s_A),r(s_B))=(10,0)$; at $\theta_2$, $(5,8)$. \rev{Then $D(\theta_1;\theta_2)=\max\{10-5,\,0-8\}=5$ and $D(\theta_2;\theta_1)=\max\{5-10,\,8-0\}=8$. Let $\varepsilon\in[5,8)$.} Any \revnew{selection rule} that relies only on a symmetric distance $\delta$ (e.g., \revnew{a Euclidean or Wasserstein distance between rewards or parameters}) is indifferent between $\theta_1$ and $\theta_2$ since $\delta(\theta_1,\theta_2)=\delta(\theta_2,\theta_1)$, and may select $\theta_2$ as the source. This fails the tolerance requirement because $D(\theta_2;\theta_1)=8>\varepsilon$. In contrast, selecting $\theta_1$ succeeds since $D(\theta_1;\theta_2)=5\le\varepsilon$.
\end{remark}

\rev{Proposition~\ref{prop:target-subopt} relates the two directional divergences to target suboptimality within a common ensemble.}
\begin{prop}[\rev{Relation to target suboptimality}]\label{prop:target-subopt}
\rev{For a common ensemble $\Pi(\theta)=\Pi(\tilde\theta)=\Pi$ \revnew{and every $\pi\in\Pi$},
\begin{equation}
0 \le J^*(\tilde\theta) - J(\tilde\theta,\revnew{\pi}) \le \revnew{J^*(\theta)-J(\theta,\pi)} + D(\theta;\tilde\theta) + D(\tilde\theta;\theta).
\end{equation}
Thus\revnew{, if} $D(\theta;\tilde\theta)\le\varepsilon$ and $D(\tilde\theta;\theta)\le\varepsilon'$, \revnew{a member that is $\xi$-suboptimal at the source is $(\xi+\varepsilon+\varepsilon')$-suboptimal} at the target \emph{within $\Pi$}\revnew{; $\xi=0$ for $\pi^*(\theta)$}.}
\end{prop}
\begin{proof}
\rev{\revnew{For every $\pi'\in\Pi$, $J(\tilde\theta,\pi')\le J(\theta,\pi')+D(\tilde\theta;\theta)\le J^*(\theta)+D(\tilde\theta;\theta)$. Taking the supremum over $\pi'$} gives $J^*(\tilde\theta)\le J^*(\theta)+D(\tilde\theta;\theta)$; subtracting $J(\tilde\theta,\pi)$ and using \revnew{$J(\theta,\pi)-J(\tilde\theta,\pi)\le D(\theta;\tilde\theta)$} gives the upper bound. \revnew{The lower bound holds because $J^*(\tilde\theta)$ is the supremum over $\Pi$.}}
\end{proof}
\vspace{-0.2em}
\subsection{Relation to CMDP Smoothness}
An alternative approach to analyzing contextual systems is to assume that the underlying MDP parameters are smooth functions of the context vector $\theta$ \cite{modi2018markov}.
\begin{definition}[Smoothness \cite{modi2018markov}]
    Given a CMDP $(\Theta, \mathcal S, \mathcal A, \mathcal M)$ where $\Theta$ is the context space, $\mathcal S$ is \revnew{a finite} state space, $\mathcal A$ is \revnew{a finite} action space, $\mathcal M$ is a function which maps a context $\theta\in\Theta$ to MDP parameters $\mathcal M(\theta)=\{p^\theta(\cdot|\cdot,\cdot),r^\theta(\cdot,\cdot),\mu^\theta(\cdot)\}$ and a distance metric over the context space $\phi(\cdot,\cdot)$, if for any two contexts $\theta_1, \theta_2\in\Theta$ \rev{and all $(s,a)$}, we have the following constraints:
    \begin{align*}
        \|p^{\theta_1}(\cdot|s,a)-p^{\theta_2}(\cdot|s,a)\|_1&\leq L_p \phi(\theta_1,\theta_2)\\
        |r^{\theta_1}(s,a)-r^{\theta_2}(s,a)|&\leq L_r \phi(\theta_1,\theta_2)
    \end{align*}
    then the CMDP is referred to as a smooth CMDP with smoothness parameters $L_p$ and $L_r$.
    \label{def:smooth-modi}
\end{definition}

The smoothness definition in \cite{modi2018markov} requires that the transition and reward functions be Lipschitz continuous over the entire state-action space, \revnew{which} can lead to pessimistic bounds if large changes in the dynamics occur in regions that are irrelevant to the policies of interest.

\begin{prop}\label{prop:divergence-vs-smoothness}
    \rev{For every $C>0$ and every discount $\gamma\in(0,1)$, there exists a discounted infinite-horizon CMDP $\{\mathcal M(\theta)\}_{\theta\in[0,1]}$ on a common $(\mathcal S,\mathcal A)$, with a context-independent initial distribution and bounded context-independent rewards, and a deterministic stationary policy set $\Pi$ such that:
    (i) $\max_{s,a}\|p^{0}(\cdot|s,a)-p^{1}(\cdot|s,a)\|_1=2$, the largest possible value;
    (ii) with $\phi(\theta,\theta')=|\theta-\theta'|$, for some state $s_{\mathrm c}$ and every $a$, the map $\theta\mapsto p^{\theta}(\cdot|s_{\mathrm c},a)$ is not Lipschitz on $[0,1]$ in $\ell_1$ and has difference quotients exceeding $C$ arbitrarily close to an interior context; and
    (iii) $D(\theta;\theta')=0$ for all $\theta,\theta'\in[0,1]$.}
\end{prop}
\begin{proof}
    \rev{Let $\mathcal S=\{s_\text{start},s_\text{safe},s_\text{c}\}$, $\mathcal A=\{a_1,a_2\}$, $\mu=\delta_{s_\text{start}}$, $r(s_\text{safe},\cdot)=1$ and $r=0$ otherwise, and $\Pi=\{\pi\}$ with the deterministic stationary policy $\pi\equiv a_1$. Independently of $\theta$, let $p(s_\text{safe}|s_\text{start},a)=1$ and $p(s_\text{safe}|s_\text{safe},a)=1$ for all $a$, so $J(\theta,\pi)=\gamma/(1-\gamma)$ for every $\theta$. Let contexts differ only at the unreachable state $s_\text{c}$: with $\sigma(\theta)=\tfrac12+\operatorname{sign}(\theta-\tfrac12)\sqrt{|\theta-1/2|/2}$ and $\operatorname{sign}(0)=0$, set $p^{\theta}(s_\text{start}|s_\text{c},a)=\sigma(\theta)$ and $p^{\theta}(s_\text{safe}|s_\text{c},a)=1-\sigma(\theta)$ for all $a$. The map $\sigma$ is nondecreasing, has $\sigma(0)=0$ and $\sigma(1)=1$, and is not Lipschitz at $\theta=1/2$; hence the endpoint transition discrepancy is $2$ and the difference quotients exceed any prescribed $C$ near $1/2$, while $D\equiv0$ because $s_\text{c}$ is never visited under $\Pi$.}
\end{proof}

This proposition demonstrates that \revnew{model smoothness} is not \revnew{necessary} for \revnew{zero-shot} generalization. Conversely, \rev{model smoothness bounds performance variation under the assumptions below.} \revnew{With finite $\mathcal S$ and $\mathcal A$, there are finitely many deterministic stationary policies, so the supremum over $\Pi$ in Lemma~\ref{lemma:model-perf} is a maximum.}
\begin{lemma}[\rev{Performance continuity from model smoothness}]\label{lemma:model-perf}
    \rev{Consider a discounted infinite-horizon CMDP with discount $\gamma\in[0,1)$, rewards bounded by $R_{\max} = \sup_{\theta,s,a}|r^\theta(s,a)|$, a context-independent initial distribution \revnew{$\mu$}, and context-independent deterministic stationary policies. If Definition~\ref{def:smooth-modi} holds with $(L_p,L_r)$, then for every policy set $\Pi$,}
    \begin{equation}
        L_{\text{perf}}(\Pi):=\sup_{\pi\in\Pi}\sup_{\theta_1\neq\theta_2}\tfrac{|J(\theta_1,\pi)-J(\theta_2,\pi)|}{\phi(\theta_1,\theta_2)}\ \le\ L_{\text{model}},
    \end{equation}
    \rev{where $L_{\text{model}} := \tfrac{L_r}{1-\gamma} + \tfrac{\gamma L_p R_{\max}}{(1-\gamma)^2}$. This is a simulation-lemma-type bound \cite{modi2018markov}, stated with explicit constants.}
\end{lemma}
\begin{proof}
    Let $\pi \in \Pi$ be an arbitrary policy\rev{, and write $V_i:=V_\pi^{\theta_i}$}. \revnew{Since $J(\theta,\pi)=\sum_s\mu(s)V_\pi^\theta(s)$ with $\mu$ context-independent, $|J(\theta_1,\pi)-J(\theta_2,\pi)|\le\|V_1-V_2\|_\infty$.} \rev{Since \revnew{the Bellman operator} $T_\pi^{\theta_1}$ is a $\gamma$-contraction in the sup norm, $\|V_1-V_2\|_\infty\le\gamma\|V_1-V_2\|_\infty+\|(T_\pi^{\theta_1}-T_\pi^{\theta_2})V_2\|_\infty$. For each $s$ and $a:=\pi(s)$, Definition~\ref{def:smooth-modi} and $\|V_2\|_\infty\le R_{\max}/(1-\gamma)$ bound $|(T_\pi^{\theta_1}V_2)(s)-(T_\pi^{\theta_2}V_2)(s)|$ by $\bigl(L_r+\gamma L_pR_{\max}/(1-\gamma)\bigr)\phi(\theta_1,\theta_2)$. Rearranging gives $\|V_1-V_2\|_\infty\le L_{\text{model}}\,\phi(\theta_1,\theta_2)$; take suprema over $\pi$ and $\theta_1\ne\theta_2$.}
\end{proof}
\begin{remark}[\rev{Context-dependent initial distributions}]\label{rem:initial}
\rev{If $\mu^\theta$ varies with $\theta$ and $\|\mu^{\theta_1}-\mu^{\theta_2}\|_1\le L_\mu\phi(\theta_1,\theta_2)$, the same argument adds $L_\mu R_{\max}/(1-\gamma)$ to $L_{\text{model}}$, since $|J(\theta_1,\pi)-J(\theta_2,\pi)|\le\|V_\pi^{\theta_1}-V_\pi^{\theta_2}\|_\infty+\|\mu^{\theta_1}-\mu^{\theta_2}\|_1\|V_\pi^{\theta_2}\|_\infty$.}
\end{remark}

Together, Proposition~\ref{prop:divergence-vs-smoothness} and \rev{Lemma}~\ref{lemma:model-perf} establish that \revnew{Assumption~\ref{assume:standing}(iii)} is weaker than the \revnew{model-smoothness} assumption of \cite{modi2018markov}\rev{, within the setting of Lemma~\ref{lemma:model-perf} with $\phi$ the Euclidean metric}.

\section{\texorpdfstring{Task\rev{-}Space Complexity}{Task-Space Complexity}}\label{sec:complexity}
\subsection{Tolerance Sets and Set Cover}
\begin{definition}[$\varepsilon$-tolerance set]\label{def:cover}
    For $\theta\in\Theta$ and a tolerance $\varepsilon>0$, its \emph{$\varepsilon$-tolerance set} is $S_\varepsilon(\theta)=\bigl\{\tilde\theta\in\Theta:\ D(\theta;\tilde\theta)\le\varepsilon\bigr\}$.
\end{definition}

\revnew{It contains the target contexts at which every controller $\pi\in\Pi(\theta)$ loses at most $\varepsilon$ relative to its source performance.}

\begin{remark}[\rev{One controller per source suffices}]\label{rem:one-policy}
\rev{By Definition~\ref{def:divergence}, $\tilde\theta\in S_\varepsilon(\theta)$ if and only if $J(\theta,\pi)-J(\tilde\theta,\pi)\le\varepsilon$ for every $\pi\in\Pi(\theta)$. Any controller \revnew{$\pi$} of the ensemble therefore \revnew{meets} the tolerance throughout $S_\varepsilon(\theta)$, with \revnew{$J(\tilde\theta,\pi)\ge J(\theta,\pi)-\varepsilon$ (and $J(\tilde\theta,\pi^*(\theta))\ge J^*(\theta)-\varepsilon$ for a best member)}, and a cover yields a library with one controller per source.}
\end{remark}

Our objective is to find the smallest set of source contexts \rev{whose} tolerance sets \rev{cover} $\Theta$. \rev{The finite-grid version is a set-cover problem. \revnew{Although} finite set cover is NP-hard \cite{karp1972reducibility}, \revnew{greedy set cover has the classical guarantee of} Theorem~\ref{thm:greedy}.}

\begin{definition}[\rev{Task-space complexity via set cover}]\label{def:complexity}
    The source task selection (STS) problem \cite{cho2024model} with $\varepsilon$-tolerance \revnew{sets} becomes \emph{a \revnew{set-cover} problem} that seeks the smallest subset of source contexts whose \revnew{tolerance sets cover} $\Theta$:
    \begin{equation*}
        \tag{P}
        \label{opt}
        \min_{\Theta_{\mathrm{src}}\subseteq\Theta} |\Theta_{\mathrm{src}}| \quad \text{s.t.}\quad  \Theta\subseteq\textstyle\bigcup_{\theta\in\Theta_{\mathrm{src}}}S_\varepsilon(\theta),
    \end{equation*}
    \rev{where the minimum is over finite subsets. The \emph{task-space complexity} at tolerance $\varepsilon$ is its optimal value,
    $\Comp_\varepsilon(\Theta;\Pi):=\min\bigl\{|\Theta_{\mathrm{src}}|:\ \Theta_{\mathrm{src}}\ \text{feasible for}\ \eqref{opt}\bigr\}$,
    with $\Comp_\varepsilon=+\infty$ if no finite cover exists (finite under Theorem~\ref{thm:bound-complexity}).}
\end{definition}

Task-space complexity $\mathsf{Comp}_\varepsilon$ is a performance-based complexity at tolerance $\varepsilon$, parameterized by the \revnew{source ensembles}. It upper-bounds the \revnew{minimum size of a library that meets} the tolerance everywhere. \rev{Enlarging $\Pi(\theta)$ strengthens the guarantee, so tolerance sets shrink and $\Comp_\varepsilon$ weakly increases, whereas pruning an ensemble to the controllers that will be deployed weakly decreases it. A singleton $\Pi(\theta_0)=\{\pi_u\}$ whose source performance is within $\varepsilon$ of its infimum over $\Theta$ covers $\Theta$ alone, so $\Comp_\varepsilon=1$.}

\subsection{\texorpdfstring{\revnew{Local Rates and Geometric Bounds}}{Local Rates and Geometric Bounds}}
\revnew{We show that the local rates of each source give an inner ball for its $\varepsilon$-tolerance set and, under a positive lower rate and localization, an outer ball.}
\begin{definition}[\rev{Local upper and lower rates}]\label{def:local-bound-rate}
    \rev{For $\theta\in\Theta$ and the fixed radius $r>0$, define, over $\tilde\theta\in (B_r(\theta)\cap\Theta)\setminus\{\theta\}$,}
    \begin{subequations}
        \begin{align}
            L^r(\theta) &:= \rev{\sup_{\tilde{\theta}} \frac{D(\theta;\tilde\theta)}{\|\theta-\tilde{\theta}\|_2}
            =} \sup_{\tilde{\theta}} \sup_{\pi \in \Pi(\theta)} \frac{J(\theta, \pi)-J(\tilde{\theta}, \pi)}{\|\theta-\tilde{\theta}\|_2},\\
            \rev{m^r(\theta)} &\rev{:= \inf_{\tilde{\theta}} \frac{D(\theta;\tilde\theta)}{\|\theta-\tilde{\theta}\|_2}
            = \inf_{\tilde{\theta}} \sup_{\pi \in \Pi(\theta)} \frac{J(\theta, \pi)-J(\tilde{\theta}, \pi)}{\|\theta-\tilde{\theta}\|_2}.}
        \end{align}
    \end{subequations}
\end{definition}

\begin{lemma}[\rev{Finiteness and local bounds}]\label{lemma:finite}
\rev{Under Assumption~\ref{assume:standing}: (i) $L^r(\theta)\le\bar\ell(\theta)<\infty$ and $m^r(\theta)\ge-\bar\ell(\theta)$; (ii) $D(\theta;\tilde\theta)\le L^r(\theta)\|\theta-\tilde\theta\|_2$ and $D(\theta;\tilde\theta)\ge m^r(\theta)\|\theta-\tilde\theta\|_2$ for all $\tilde\theta\in B_r(\theta)\cap\Theta$; (iii) $\tilde\theta\mapsto D(\theta;\tilde\theta)$ is $\bar\ell(\theta)$-Lipschitz, hence continuous, on $B_r(\theta)\cap\Theta$.}
\end{lemma}
\begin{proof}
\rev{(i) For every $\pi\in\Pi(\theta)$ and $\tilde\theta\in B_r(\theta)\cap\Theta$, $|J(\theta,\pi)-J(\tilde\theta,\pi)|\le\ell_\pi(\theta)\|\theta-\tilde\theta\|_2\le\bar\ell(\theta)\|\theta-\tilde\theta\|_2$; take the \revnew{supremum over $\Pi(\theta)$} and the supremum/infimum over $\tilde\theta$. (ii) is immediate from Definition~\ref{def:local-bound-rate}. (iii) $D(\theta;\cdot)$ is a \revnew{supremum of $\bar\ell(\theta)$-Lipschitz functions and is finite on $B_r(\theta)\cap\Theta$ by (i)}.}
\end{proof}

\rev{The upper rate bounds each ensemble member's loss per unit distance within radius $r$, yielding an inner \revnew{ball in which} all ensemble members meet the tolerance. A positive lower rate means that at each non-source target within that radius, \revnew{the divergence is} at least $m^r(\theta)$ times the distance\revnew{, although the members that cause it} may differ across targets.}

\rev{A single controller cannot have a positive lower rate at a differentiable interior source, but an ensemble can, because \revnew{the members with the largest losses} may change with the target.}
\begin{prop}[\rev{Sign of the lower rate}]\label{prop:sign-mr}
\rev{Let $\theta\in\operatorname{int}\Theta$ and suppose each $J(\cdot,\pi)$, $\pi\in\Pi(\theta)$, is differentiable at $\theta$ with gradient $g_\pi:=\nabla_{\tilde\theta}J(\tilde\theta,\pi)\big|_{\tilde\theta=\theta}$\revnew{, uniformly in $\pi$: $\sup_{\pi\in\Pi(\theta)}|J(\theta+h,\pi)-J(\theta,\pi)-g_\pi^\top h|=o(\|h\|_2)$ (automatic for a finite ensemble)}.
(a) \revnew{If $0\notin\operatorname{int}\operatorname{conv}\{g_\pi:\pi\in\Pi(\theta)\}$, in particular if $|\Pi(\theta)|\le d$, then $m^r(\theta)\le 0$; for a singleton $\Pi(\theta)=\{\pi\}$, the inequality is strict whenever $g_\pi\neq0$.}
(b) If $0\in\operatorname{int}\operatorname{conv}\{g_\pi:\pi\in\Pi(\theta)\}$, and $\rho(\theta)>0$ denotes the radius of the largest Euclidean ball centered at the origin contained in \revnew{the closed convex hull $\overline{\operatorname{conv}}\{-g_\pi:\pi\in\Pi(\theta)\}$}, then
$\liminf_{\tilde\theta\to\theta} D(\theta;\tilde\theta)/\|\tilde\theta-\theta\|_2 \ \ge\ \rho(\theta)>0$.
If, in addition, $D(\theta;\tilde\theta)>0$ for every $\tilde\theta\in(B_r(\theta)\cap\Theta)\setminus\{\theta\}$, then $m^r(\theta)>0$.}
\revnew{(c) If $0\in\operatorname{int}\operatorname{conv}\{g_\pi\}$ and $|J(\theta+h,\pi)-J(\theta,\pi)-g_\pi^\top h|\le\beta\|h\|_2^2/2$ for all $\pi\in\Pi(\theta)$ and all $\theta+h\in B_r(\theta)\cap\Theta$ (e.g., $\beta$-Lipschitz gradients on a convex $\Theta$), then $m^r(\theta)\ge\rho(\theta)-\beta r/2$, which is positive for $r<2\rho(\theta)/\beta$.}
\end{prop}
\begin{proof}
\rev{\revnew{Uniform differentiability gives, uniformly over unit $u$, $D(\theta;\theta+tu)/t=\sup_{\pi}\bigl(-g_\pi^\top u\bigr)+o(1)$ as $t\downarrow0$, so $m^r(\theta)\le\sup_\pi(-g_\pi^\top u)$ for every unit $u$. (a) Since $0\notin\operatorname{int}\operatorname{conv}\{g_\pi\}$, a separating or supporting hyperplane at $0$ gives a unit $u$ with $g_\pi^\top u\ge0$ for all $\pi$, so $m^r(\theta)\le0$; at most $d$ points have a convex hull with empty interior. For a singleton with $g_\pi\ne0$, $u=g_\pi/\|g_\pi\|_2$ gives $m^r(\theta)\le-\|g_\pi\|_2<0$.}
(b) \revnew{The supremum is the support function of $\overline{\operatorname{conv}}\{-g_\pi\}$} at $u$, which is at least $\rho(\theta)$ for every unit $u$ by the ball-containment hypothesis, proving the liminf claim. For the last claim, suppose $m^r(\theta)\le0$ and take $\tilde\theta_k$ with ratios converging to $m^r(\theta)$ and, by compactness, $\tilde\theta_k\to\bar\theta\in B_r(\theta)\cap\Theta$. If $\bar\theta\ne\theta$, continuity (Lemma~\ref{lemma:finite}(iii)) gives $D(\theta;\bar\theta)=\lim_k D(\theta;\tilde\theta_k)\le 0$, contradicting positivity; if $\bar\theta=\theta$, the ratios are eventually at least $\rho(\theta)/2>0$ by the liminf claim, again a contradiction.} \revnew{(c) For $\tilde\theta=\theta+h$, each member gives $J(\theta,\pi)-J(\tilde\theta,\pi)\ge-g_\pi^\top h-\beta\|h\|_2^2/2$, and by (b) the supremum of $-g_\pi^\top h$ over $\pi$ is at least $\rho(\theta)\|h\|_2$, so $D(\theta;\tilde\theta)/\|h\|_2\ge\rho(\theta)-\beta r/2$.}
\end{proof}

\rev{\revnew{Under the hypotheses of Proposition~\ref{prop:sign-mr}}, the condition in \revnew{part} (b), positive spanning of the gradients, means that in every direction away from the source some ensemble member has a negative directional derivative, so the divergence grows at first order in every direction. \revnew{By Proposition~\ref{prop:sign-mr}(a), the condition is also necessary for a positive lower rate, which therefore requires at least $d+1$ members.}} \rev{Example~\ref{ex:spanning} gives a family where both \revnew{the inner and the outer ball} apply.} \rev{At a boundary source only inward directions are feasible\revnew{. Given one-sided differentiability, the proof of Proposition~\ref{prop:sign-mr}(b) carries over with $u$ restricted to these directions}, so positive spanning still gives positive local growth, and even a single controller can have a positive lower rate, as $\pi_0^+$ does at $\theta=0$ in that example.} \rev{The inner-ball experiments in Section~\ref{sec:experiments} do not require positive spanning; their grid-restricted lower rates do not establish positive continuum rates.}

\rev{We say $\theta$ is \emph{$(\varepsilon,r)$-localized} if $D(\theta;\tilde\theta)>\varepsilon$ for all $\tilde\theta\in\Theta\setminus B_r(\theta)$, i.e., the tolerance set does not extend beyond the scale $r$ at which the local rates are computed.}
\begin{lemma}[Geometric \revnew{bounds on the} $\varepsilon$-tolerance \revnew{set}]\label{lemma:gen-bounds}
    \rev{Let $\varepsilon>0$ and define the certified inner radius
    $\rho_{\mathrm{in}}(\theta):=r$ if $L^r(\theta)\le0$ and $\rho_{\mathrm{in}}(\theta):=\min\{r,\ \varepsilon/L^r(\theta)\}$ otherwise. Then}
    \begin{equation}
        \rev{B_{\rho_{\mathrm{in}}(\theta)}(\theta)\cap\Theta\ \subseteq\ S_\varepsilon(\theta).}
    \end{equation}
    \rev{If $m^r(\theta)>0$, then $S_\varepsilon(\theta)\cap B_r(\theta)\subseteq B_{\varepsilon/m^r(\theta)}(\theta)$; if $\theta$ is in addition $(\varepsilon,r)$-localized, then $S_\varepsilon(\theta)\subseteq B_{\varepsilon/m^r(\theta)}(\theta)$.}
\end{lemma}
\begin{proof}
\rev{Inner: for $\tilde\theta\in B_{\rho_{\mathrm{in}}(\theta)}(\theta)\cap\Theta$, Lemma~\ref{lemma:finite}(ii) gives $D(\theta;\tilde\theta)\le L^r(\theta)\|\theta-\tilde\theta\|_2\le\varepsilon$ (the case $L^r(\theta)\le0$ gives $D\le0\le\varepsilon$ directly, and $D(\theta;\theta)=0$). Outer: for $\tilde\theta\in S_\varepsilon(\theta)\cap B_r(\theta)$, $\varepsilon\ge D(\theta;\tilde\theta)\ge m^r(\theta)\|\theta-\tilde\theta\|_2$, so $\|\theta-\tilde\theta\|_2\le\varepsilon/m^r(\theta)$; under $(\varepsilon,r)$-localization, $S_\varepsilon(\theta)\subseteq B_r(\theta)$ and the containment is unconditional.}
\end{proof}

\begin{theorem}[Bounds on \revnew{task}\rev{-}\revnew{space complexity}]
\label{thm:bound-complexity}
\rev{Let Assumption~\ref{assume:standing} hold and let $\varepsilon>0$. Suppose $L_{\max} := \sup_{\theta\in\Theta} L^r(\theta) < \infty$ and set $\rho:=r$ if $L_{\max}\le0$, else $\rho := \min\{r,\varepsilon/L_{\max}\}$. Then $\Comp_\varepsilon$ is finite and
\begin{equation}
    \Comp_\varepsilon\ \le\ \frac{\mathrm{Vol}\bigl(\Theta\oplus(\rho/2)B_1\bigr)}{v_d\,(\rho/2)^d},
    \label{eq:upper}
\end{equation}
where $\oplus$ is the Minkowski sum, $B_1$ the unit ball, and $v_d=\pi^{d/2}/\Gamma(d/2+1)$ the unit-ball volume. If, in addition, every $\theta\in\Theta$ is $(\varepsilon,r)$-localized and $m_{\min}:=\inf_{\theta\in\Theta} m^r(\theta)>0$, then with $R:=\min\{r,\varepsilon/m_{\min}\}$,}
\begin{equation}
    \rev{\Comp_\varepsilon\ \ge\ \frac{\mathrm{Vol}(\Theta)}{v_d\,R^{\,d}}.}
    \label{eq:lower}
\end{equation}
\end{theorem}
\begin{proof}
    \rev{\emph{Upper bound.} Let $\{\theta_1,\dots,\theta_N\}\subseteq\Theta$ be a maximal $\rho$-separated set (pairwise distances $\ge\rho$; one of maximum cardinality exists and is finite because $\rho$-packings of a compact set have bounded size). By maximality, every $\theta\in\Theta$ satisfies $\|\theta-\theta_i\|_2<\rho$ for some $i$, and since $L^r(\theta_i)\le L_{\max}$ implies $\rho\le\rho_{\mathrm{in}}(\theta_i)$, Lemma~\ref{lemma:gen-bounds} gives $\theta\in S_\varepsilon(\theta_i)$; hence $\{\theta_i\}$ is feasible for \eqref{opt} and $\Comp_\varepsilon\le N$. The balls $B_{\rho/2}(\theta_i)$ have pairwise disjoint interiors and lie in $\Theta\oplus(\rho/2)B_1$, so $N\, v_d(\rho/2)^d\le \mathrm{Vol}(\Theta\oplus(\rho/2)B_1)$, which is \eqref{eq:upper}.
    \emph{Lower bound.} If $\{\theta_1,\dots,\theta_N\}$ is feasible, localization gives $S_\varepsilon(\theta_i)\subseteq B_r(\theta_i)$ and Lemma~\ref{lemma:gen-bounds} gives $S_\varepsilon(\theta_i)\subseteq B_{\varepsilon/m_{\min}}(\theta_i)$, so $\Theta\subseteq\bigcup_i S_\varepsilon(\theta_i)\subseteq\bigcup_i B_{R}(\theta_i)$, and subadditivity of volume gives $\mathrm{Vol}(\Theta)\le N v_d R^d$.}
\end{proof}

Each \rev{source ensemble} certifies \revnew{an inner ball in which} \rev{all its controllers meet the tolerance}. The ratio of context-space volume to the volume of one \revnew{such ball} tells us roughly how many \rev{sources} are needed. \rev{Theorem~\ref{thm:bound-complexity} makes this count explicit with standard packing arguments, but its radii come from ensemble-dependent \revnew{local} rates. Unlike Cover-Rmax~\cite{modi2018markov}, whose radius is fixed by global model constants, the certified radii vary with the local rate}\rev{; the upper bound \eqref{eq:upper} uses only their uniform lower bound $\rho$, whereas greedy selection exploits their variation. For a box, the bounds take an explicit form.}
\begin{corollary}[\rev{Scaling for a box}]\label{cor:box}
\rev{Let $\Theta=\prod_{i=1}^d[a_i,b_i]$ with $w_i:=b_i-a_i>0$, let $L_{\max}\in(0,\infty)$, and let $\varepsilon\le rL_{\max}$. Then $\Comp_\varepsilon\le v_d^{-1}(2L_{\max}/\varepsilon)^d\prod_{i}(w_i+\varepsilon/L_{\max})$. If, in addition, the lower-bound hypotheses of Theorem~\ref{thm:bound-complexity} hold and $\varepsilon\le r\,m_{\min}$, then $\Comp_\varepsilon\ge v_d^{-1}(m_{\min}/\varepsilon)^d\prod_{i}w_i$. Both bounds scale as $\varepsilon^{-d}$, and they differ by at most the factor $(2L_{\max}/m_{\min})^d\prod_i\bigl(1+\varepsilon/(L_{\max}w_i)\bigr)$.}
\end{corollary}
\begin{proof}
\rev{Since $B_1\subseteq[-1,1]^d$, $\Theta\oplus(\rho/2)B_1\subseteq\prod_i[a_i-\rho/2,\,b_i+\rho/2]$, whose volume is $\prod_i(w_i+\rho)$. The hypotheses $\varepsilon\le rL_{\max}$ and $\varepsilon\le r\,m_{\min}$ give $\rho=\varepsilon/L_{\max}$ and $R=\varepsilon/m_{\min}$; substitute them into \eqref{eq:upper} and \eqref{eq:lower}.}
\end{proof}

\begin{example}[\rev{A one-dimensional example}]\label{ex:spanning}
\rev{Let $\Theta=[0,1]$, $g>0$, $r=\varepsilon/g$, and $\Pi(\theta)=\{\pi_\theta^{\pm}\}$ with $J(\tilde\theta,\pi_\theta^{\pm})=1\mp g(\tilde\theta-\theta)$. Then $D(\theta;\tilde\theta)=g|\tilde\theta-\theta|$, $L^r\equiv m^r\equiv g$, and every $\theta$ is $(\varepsilon,r)$-localized. At an interior source, the gradients $\mp g$ positively span $\mathbb R$ with $\rho(\theta)=g$. The tolerance sets are intervals of radius $\varepsilon/g$, clipped to $[0,1]$, so $\Comp_\varepsilon=\lceil g/(2\varepsilon)\rceil$. Corollary~\ref{cor:box} gives $g/(2\varepsilon)\le\Comp_\varepsilon\le g/\varepsilon+1$, with the lower bound tight up to rounding.}
\end{example}

\rev{In practice we compute covers on a finite grid. The following lemma gives a continuum guarantee when certified inner balls, shrunk by the grid covering radius, cover the grid.}
\begin{lemma}[\rev{Grid-to-continuum transfer}]\label{lemma:grid}
\rev{Let $\Theta_{\mathrm{grid}}\subseteq\Theta$ be a finite grid such that every $\theta\in\Theta$ has a grid point within distance $\delta$ (for a hypercubic grid of spacing $w$ on a box, the center of a cell is farthest from the grid, so $\delta=w\sqrt d/2$). Let sources $\{\theta_i\}\subseteq\Theta$ have certified inner radii $\rho_i\ge\delta$, i.e., $B_{\rho_i}(\theta_i)\cap\Theta\subseteq S_\varepsilon(\theta_i)$. If $\Theta_{\mathrm{grid}}\subseteq\bigcup_i B_{\rho_i-\delta}(\theta_i)$, then $\Theta\subseteq\bigcup_i S_\varepsilon(\theta_i)$.}
\end{lemma}
\begin{proof}
\rev{For $\theta\in\Theta$, pick a grid point $g$ with $\|\theta-g\|_2\le\delta$ and an index $i$ with $\|g-\theta_i\|_2\le\rho_i-\delta$. The triangle inequality gives $\|\theta-\theta_i\|_2\le\rho_i$, so $\theta\in B_{\rho_i}(\theta_i)\cap\Theta\subseteq S_\varepsilon(\theta_i)$.}
\end{proof}
\subsection{\texorpdfstring{\revnew{Greedy Source Selection}}{Greedy Source Selection}}\label{subsec:greedy}
\rev{Algorithm~\ref{alg:main_alg} applies classical greedy set cover using either of two tolerance-set oracles $\widehat S_\varepsilon$:}
\begin{itemize}
    \item \rev{\emph{Grid-exact oracle}: $\widehat S_\varepsilon(\theta)=S_\varepsilon(\theta)\cap\Theta_{\mathrm{grid}}$ by enumeration, i.e., evaluating $D(\theta;\tilde\theta)$ at every grid context (expensive; Proposition~\ref{prop:eval-cost});}
    \item \rev{\emph{Rate-ball oracle}: $\widehat S_\varepsilon(\theta)=B_{\hat\rho(\theta)-\delta'}(\theta)\cap\Theta_{\mathrm{grid}}$ if $\hat\rho(\theta)\ge\delta'$, else $\emptyset$, where $\bar L(\theta)$ is a rate bound or estimate, $\hat\rho(\theta)=\min\{r,\varepsilon/\bar L(\theta)\}$ ($\hat\rho:=r$ if $\bar L(\theta)\le0$)\revnew{,} and $\delta'\in\{0,\delta\}$ is a shrinkage parameter. With a valid bound $\bar L(\theta)\ge L^r(\theta)$, choosing $\delta'=\delta$ allows grid-to-continuum transfer by Lemma~\ref{lemma:grid}. For the empirical estimate of Section~\ref{subsec:estimation}, we use $\delta'=0$ and assess coverage on the grid.}
\end{itemize}
\rev{Exact evaluations or a valid rate bound \revnew{ensures} $\widehat S_\varepsilon(\theta)\subseteq S_\varepsilon(\theta)\cap\Theta_{\mathrm{grid}}$, so a returned cover meets the grid tolerance. Grid-exact enumeration alone does not certify \revnew{contexts} between grid \revnew{points}. With $\Theta_{\mathrm{cand}}=\Theta_{\mathrm{grid}}$, Algorithm~\ref{alg:main_alg} always returns a cover with the grid-exact oracle or with $\delta'=0$, since every candidate covers itself, whereas the rate-ball oracle with $\delta'=\delta$ may report infeasibility.}


\begin{algorithm}[t]
    \caption{\revnew{Greedy source selection}}
    \label{alg:main_alg}
    \begin{algorithmic}[1]
        \REQUIRE \rev{coverage grid $\Theta_{\mathrm{grid}}$, candidate set $\Theta_{\mathrm{cand}}\subseteq\Theta_{\mathrm{grid}}$, tolerance $\varepsilon$, oracle $\widehat S_\varepsilon$ (grid-exact or rate-ball)}
        \STATE \rev{for each $\theta\in\Theta_{\mathrm{cand}}$: synthesize/train $\Pi(\theta)$ and compute $\widehat S_\varepsilon(\theta)$}
        \STATE Initialize $U\gets\rev{\Theta_{\mathrm{grid}}}$; $\Theta_{\mathrm{src}}\gets\emptyset$
        \WHILE {$U\neq\emptyset$}
            \STATE \rev{$\theta^*\gets\arg\max_{\theta \in \Theta_{\mathrm{cand}}} |U\cap\widehat{S}_\varepsilon(\theta)|$}
            \STATE \rev{\textbf{if} $U\cap\widehat{S}_\varepsilon(\theta^*)=\emptyset$ \textbf{then return} \emph{infeasible}}
            \STATE \rev{$U \gets U \setminus \widehat{S}_\varepsilon(\theta^*)$;\quad $\Theta_{\mathrm{src}}\gets \Theta_{\mathrm{src}} \cup \{\theta^*\}$}
        \ENDWHILE
        \RETURN $\Theta_{\mathrm{src}}$ \rev{(deploy $\pi^*(\theta)$ for each $\theta\in\Theta_{\mathrm{src}}$)}
    \end{algorithmic}
\end{algorithm}

\begin{theorem}[Greedy guarantee\rev{; classical}]\label{thm:greedy}
    For a \rev{grid} of size $\rev{N} = \rev{|\Theta_{\mathrm{grid}}|}$ \rev{and a feasible instance}, Algorithm~\ref{alg:main_alg} terminates after selecting at most $\rev{\mathrm{OPT}} \cdot (\ln \rev{N} + 1)$ \rev{sources}, where $\rev{\mathrm{OPT}}$ is the \rev{optimal cover size for the instance defined by the sets $\{\widehat S_\varepsilon(\theta)\}_{\theta\in\Theta_{\mathrm{cand}}}$}. \rev{This is the classical greedy set-cover guarantee \cite{vazirani2001approximation}, using $H(N)=\sum_{i=1}^N 1/i\le\ln N+1$.}
\end{theorem}

\subsection{\texorpdfstring{\rev{Estimating the Local Rate: Certified vs. Empirical}}{Estimating the Local Rate: Certified vs. Empirical}}\label{subsec:estimation}
\revnew{The inner ball of} \rev{Lemma~\ref{lemma:gen-bounds}} \revnew{requires} the local upper rate $L^r(\theta)$. \rev{For feasible radii $\zeta\in\mathcal Z_r(\theta):=\{\zeta\in(0,r]:\exists\tilde\theta\in\Theta,\ \|\theta-\tilde\theta\|_2=\zeta\}$\revnew{,} let \revnew{$f(\theta,\zeta):=\sup_{\pi\in\Pi(\theta)}\sup_{\tilde\theta\in\Theta,\ \|\theta-\tilde\theta\|_2=\zeta}\bigl(J(\theta,\pi)-J(\tilde\theta,\pi)\bigr)$} and $h=f/\zeta$, so that $L^r(\theta)=\sup_{\zeta\in\mathcal Z_r(\theta)}h(\theta,\zeta)$.} Directly solving this nested supremum is intractable. \rev{Algorithm~\ref{alg:L_max_approximation} uses multi-start projected zeroth-order finite-difference ascent \cite{spall1992multivariate}. Its surrogate $\hat h$ is the maximum sampled ratio of divergence to actual source--target distance, over at most $n_{\mathrm{dir}}$ non-source targets. Each $J$ is evaluated directly or estimated from $n_{\mathrm{roll}}$ rollouts. The sampled maximization and potentially nonsmooth $h$ make this a heuristic; the maximum over all probes and restarts is the empirical estimate $\hat L^r(\theta)$.}

\begin{algorithm}[t]
  \caption{\rev{Multi-start zeroth-order estimate of $L^r(\theta)$}}
  \label{alg:L_max_approximation}
  \begin{algorithmic}[1]
    \REQUIRE \rev{$\theta$, $r$, search interval $[\zeta_{\min},r]$, restarts $S$, iterations $K$, gains $\eta_t,\sigma_t$, $n_{\mathrm{dir}}$ targets, $n_{\mathrm{roll}}$ rollouts per $J$-estimate}
    \FOR{\rev{$s=1$} \TO \rev{$S$}}
      \STATE \rev{sample $\zeta\in[\zeta_{\min},r]$}
      \FOR{$t=1$ \TO \rev{$K$}}
        \STATE \rev{$\hat h^{\pm}\gets\hat h\bigl(\theta,\mathrm{proj}_{[\zeta_{\min},r]}(\zeta\pm\sigma_t)\bigr)$ over sampled feasible targets}
        \STATE \rev{$\zeta\gets\mathrm{proj}_{[\zeta_{\min},r]}\bigl(\zeta+\eta_t(\hat h^+-\hat h^-)/(2\sigma_t)\bigr)$}
      \ENDFOR
      \STATE \rev{evaluate $\hat h(\theta,\zeta)$ at the final iterate}
      \STATE \rev{$\hat L_s\gets$ maximum sampled $\hat h(\theta,\zeta)$ over restart $s$}
    \ENDFOR
    \RETURN \rev{$\hat L^r(\theta):=\max_s\hat L_s$}
  \end{algorithmic}
\end{algorithm}

\rev{\textbf{Certified vs. empirical rate bounds.} Any valid upper bound $\bar L(\theta)\ge L^r(\theta)$ gives a certified inner ball: the rate-ball radius satisfies $\hat\rho(\theta)\le\rho_{\mathrm{in}}(\theta)$ (Lemma~\ref{lemma:gen-bounds}).} \rev{Rate bounds or estimates can be obtained in three ways. (i) \emph{\revnew{Perturbation analysis}}: for the finite-horizon LQR family of Section~\ref{sec:experiments}, \revnew{matrix-exponential or Lyapunov perturbation bounds \cite{zhou1998essentials} can in principle give a certified rate bound}. (ii) \emph{Model smoothness}: within Lemma~\ref{lemma:model-perf}'s setting, known constants $(L_p,L_r)$ for the Euclidean metric yield the \revnew{certified}, possibly conservative bound $\bar L\equiv L_{\text{model}}$. (iii) \emph{Black-box estimation}: Algorithm~\ref{alg:L_max_approximation} can underestimate $L^r$, so its balls are empirical rate balls, not \revnew{certified inner balls}. Without additional regularity assumptions, finitely many black-box evaluations cannot yield a nontrivial uniform upper bound over an unrestricted function class. We inflate the estimate to $(1+\kappa)\hat L^r$ and} \rev{test the resulting balls at grid contexts and at held-out contexts (Section~\ref{subsec:validity}).}

\begin{prop}[\rev{Evaluation cost of \revnew{the two oracles}}]\label{prop:eval-cost}
Let $\Theta_{\text{grid}}\subset\Theta$ have $N$ contexts and let $|\Pi(\theta)|=P$ for a source $\theta$. Suppose evaluating $J(\tilde\theta,\pi)$ to a fixed accuracy has unit cost. \revnew{For this source, the grid-exact oracle} incurs $\Theta(PN)$ \revnew{controller--context} evaluations in the worst case. \rev{The \revnew{rate-ball oracle with} Algorithm~\ref{alg:L_max_approximation} costs $O(S K\, n_{\mathrm{dir}}\, P)$ evaluations, independent of $N$.}
\end{prop}
\begin{proof}
\rev{Enumeration evaluates all $P$ controllers at all $N$ grid contexts. Each restart makes $2K+1$ calls to $\hat h$, each using at most $n_{\mathrm{dir}}P$ target evaluations. \revnew{Counting the $P$ source evaluations once}, the total is at most $P+S(2K+1)n_{\mathrm{dir}}P$.}
\end{proof}

\rev{\textbf{Cost of the full pipeline.} Let $M=|\Theta_{\mathrm{cand}}|$, $N=|\Theta_{\mathrm{grid}}|$, $P=|\Pi(\theta)|$, $\tau=|\Theta_{\mathrm{src}}|$, and let $C_{\mathrm{train}}$ and $C_{\mathrm{roll}}$ denote per-controller training and per-rollout costs, with $n_{\mathrm{roll}}$ rollouts per $J$-estimate. Algorithm~\ref{alg:main_alg} with the rate-ball oracle costs $MP C_{\mathrm{train}}+O(MSKn_{\mathrm{dir}}Pn_{\mathrm{roll}})C_{\mathrm{roll}}+O(\tau MN)$, while the grid-exact middle term is $\Theta(MPNn_{\mathrm{roll}})C_{\mathrm{roll}}$. Because Algorithm~\ref{alg:main_alg} trains all candidate ensembles before selection, $\tau$ counts deployed controllers, not trained ones.}

\rev{\textbf{Rollout accuracy.} For each \revnew{evaluated} controller--context pair, assume fresh independent rollouts with conditionally $\sigma$-sub-Gaussian returns. With at most $N_{\mathrm{eval}}$ \revnew{evaluations}, $n_{\mathrm{roll}}\ge 2\sigma^2\varepsilon_0^{-2}\log(2N_{\mathrm{eval}}/\beta)$ ensures simultaneous $\varepsilon_0$-accuracy with probability at least $1-\beta$, by a sub-Gaussian tail bound for each \revnew{evaluation} and a union bound. Each estimated divergence is then within $2\varepsilon_0$ of the true one, so for $0<2\varepsilon_0<\varepsilon$, tightening the grid-exact divergence threshold to $\varepsilon-2\varepsilon_0$ keeps every accepted target within the tolerance.}

\section{Numerical Experiments}\label{sec:experiments}
\rev{This section illustrates the results of Section~\ref{sec:complexity} on two systems: whether empirical rate balls stay inside the tolerance sets (Lemma~\ref{lemma:gen-bounds}), how many sources greedy selection needs (Theorem~\ref{thm:greedy}), when the lower bound of Theorem~\ref{thm:bound-complexity} applies (Proposition~\ref{prop:sign-mr}), and the evaluation costs of Proposition~\ref{prop:eval-cost}.} First, we use a continuous-time linear time-invariant \revnew{mass--spring--damper} (MSD) system with \rev{state-only} LQR controllers\revnew{, whose performance is available in closed form}. Second, we test our approach on the discrete-time nonlinear CartPole system with \rev{context-aware} deep RL controllers\revnew{, whose performance is available only through rollouts}.

\subsection{Experimental Setup}
\paragraph{MSD system with LQR controller}
The MSD system consists of a mass $m$ connected to a spring with stiffness $k$ and a damper with damping coefficient $c$. \rev{With state $\mathbf{x}=[x\ \dot x]^\top$ and force input $u$, $\dot{\mathbf{x}}=A\mathbf{x}+Bu$ with $A=\bigl[\begin{smallmatrix}0&1\\-k/m&-c/m\end{smallmatrix}\bigr]$ and $B=\bigl[\begin{smallmatrix}0\\1/m\end{smallmatrix}\bigr]$.} The context vector is $\theta = [k, c]^\top$, while the mass is held constant at $m=1.0$. The context space $\Theta$ is \rev{sampled on} a grid where both $k$ and $c$ range from $0.1$ to $8.0$, discretized into 100 points each. We set the LQR weighting matrices to $Q=\operatorname{diag}(10,4)$ and $R = [1]$. We create an ensemble of controllers by adding Gaussian noise to the optimal gain\rev{: at each source, $\Pi(\theta)$ has $P=30$ members $K^*(\theta)+\Delta_i$, where $K^*(\theta)$ is the infinite-horizon LQR gain and the $\Delta_i\sim\mathcal N(0,0.04I_2)$ are independent (seed 42).} \rev{Performance is the negative \revnew{of the} finite-horizon quadratic cost over $T=20$~s from $x_0=[1,0]^\top$, evaluated in closed form. For each trial radius, targets sampled along random directions are clipped to the context box; zero-distance targets are discarded and ratios use the actual distances. The rate estimates use $\varepsilon=2$, $r=2$, $\zeta_{\min}=\delta=0.05643$, $S=5$, $K=15$, $n_{\mathrm{dir}}=10$, $\kappa=0.1$, raw Euclidean coordinates $(k,c)$, $\bar L=(1+\kappa)\hat L^r$, and $\delta'=0$}. \rev{Algorithm~\ref{alg:L_max_approximation} uses the gains $\eta_t=0.3(r-\zeta_{\min})(t+1)^{-0.602}$ and $\sigma_t=0.1(r-\zeta_{\min})(t+1)^{-0.101}$, and its two probes sample their targets independently.}

\paragraph{CartPole with \revnew{deep RL} controller}
We use CARL, a CMDP variant of the standard \rev{CartPole} implementation from OpenAI Gym \cite{benjamins2023contextualize}. The context vector is $\theta = [\text{pole length}, \text{cart mass}]^\top$. The context space $\Theta$ is \rev{sampled on} a $20 \times 20$ grid where the pole length ranges from $0.25$ to $5.0$\rev{~m} and the cart mass from $0.5$ to $10.0$\rev{~kg}. For each \revnew{context}, we train the policy for \rev{five} random seeds \rev{(0--4)} using Proximal Policy Optimization (PPO) \cite{schulman2017proximal} with \rev{Stable-Baselines3 2.5.0}~\cite{raffin2021stable}, to create the \revnew{ensemble} \rev{$\Pi(\theta)$, using default hyperparameters, $10^6$ steps, and the best checkpoint by five-episode evaluations every $10^4$ steps.} \rev{The feedforward neural networks take 12 inputs: four physical-state coordinates and eight CARL context features, including the target length and mass.} The performance $J(\theta, \pi)$ is the expected cumulative reward (balancing time) over an episode\rev{, estimated by the mean return over 50 deterministic-policy test episodes with reset seeds 0--49.} \rev{The CartPole study} \rev{uses $\varepsilon=100$, $r=1$, $\kappa=0.1$, $\delta'=0$, the estimator settings of the MSD study except $\zeta_{\min}=0.25$, and the raw Euclidean norm in the stated physical units; the sampled targets at a trial radius $\zeta$ are grid contexts whose distance from the source lies within $w\sqrt2/2$ of $\zeta$, with $w=0.5$ the larger grid spacing.}

\subsection{Signed Divergence and its Underapproximation}\label{subsec:validity}
Fig.~\ref{fig:cartpole_landscape}\rev{(a)} plots the signed divergence landscape for a single source \rev{ensemble} in CartPole. The true $\varepsilon$-tolerance set corresponds to the region where this surface is below \rev{$\varepsilon$}. \rev{In Fig.~\ref{fig:cartpole_landscape}(b), this set extends along the cart-mass axis well beyond the empirical rate ball, which contains 15 grid points.}

\begin{figure}[!t]
    \centering
    \includegraphics[width=0.9\columnwidth]{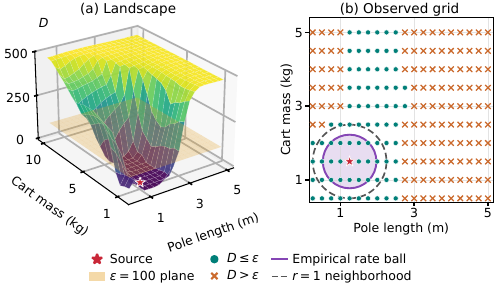}
    \caption{\rev{CartPole at source $(1.25\,\mathrm{m},1.5\,\mathrm{kg})$, using the maximum performance loss over five context-aware controllers trained with PPO. (a) Signed divergence and tolerance $\varepsilon=100$; the surface interpolates stored grid data only for visualization. (b) Stored grid data and empirical \revnew{rate} ball.}}
    \label{fig:cartpole_landscape}
\end{figure}

\rev{The inner ball of Lemma~\ref{lemma:gen-bounds} requires a valid rate bound.} \rev{Algorithm~\ref{alg:L_max_approximation} can underestimate $L^r$, so an empirical \revnew{rate} ball may contain contexts with $D>\varepsilon$. On the grid, 3 of the $10^4$ MSD balls do, and shrinking them to grid-valid radii leaves the greedy count at 139; no CartPole ball does. Because most CartPole balls contain no other grid point, we test 20 held-out contexts drawn uniformly in each selected ball and clipped to $\Theta$, evaluated in closed form for MSD and by new 50-episode rollouts for CartPole. All $2780$ MSD contexts meet the tolerance, and 4557 of 4580 CartPole contexts do (worst $1.37\,\varepsilon$). Neither test certifies continuum coverage.}

\subsection{\texorpdfstring{\revnew{Source Selection}}{Source Selection}}\label{subsec:selection}
We compare \revnew{greedy selection with} \rev{empirical rate balls} against several baselines: \revnew{grid-exact enumeration (tolerance sets computed} \rev{on the grid}\revnew{)}, a uniform \rev{lattice} selection, and a random selection \rev{from still-uncovered contexts, the last two with the same rate balls. Counting one deployed controller per selected source, greedy \revnew{selection} needs 139 controllers on MSD, 53\% fewer than the random median of 296, while uniform lattices need 1156; with $\delta'=\delta$, greedy and random \revnew{selection} need 211 and $472\pm23$. On CartPole, greedy \revnew{selection} needs 229 versus a random mean of $290.2\pm6.1$ (21\% fewer), and only the full lattice covers the grid (Fig.~\ref{fig:sample_efficiency_alg}).}

\begin{figure}[!t]
    \centering
    \subfloat[\revnew{Mass--spring--damper}]{%
        \includegraphics[width=0.48\columnwidth]{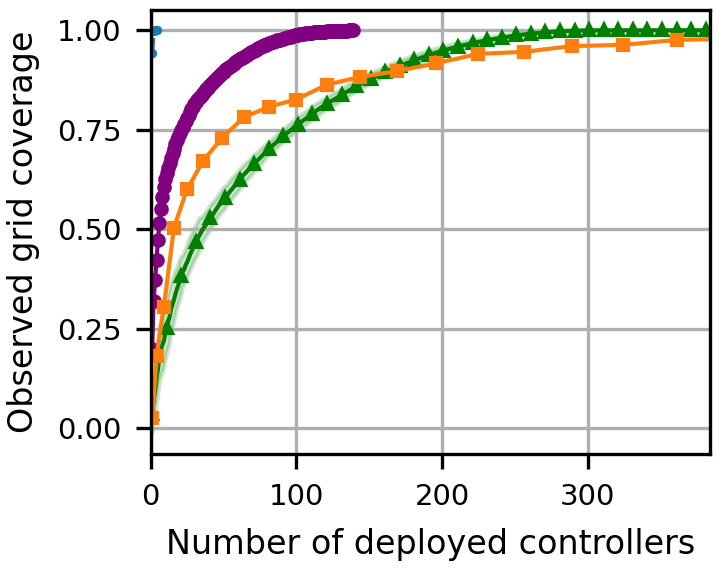}%
        \label{subfig:msd_efficiency}}%
    \hfill
    \subfloat[CartPole]{%
        \includegraphics[width=0.48\columnwidth]{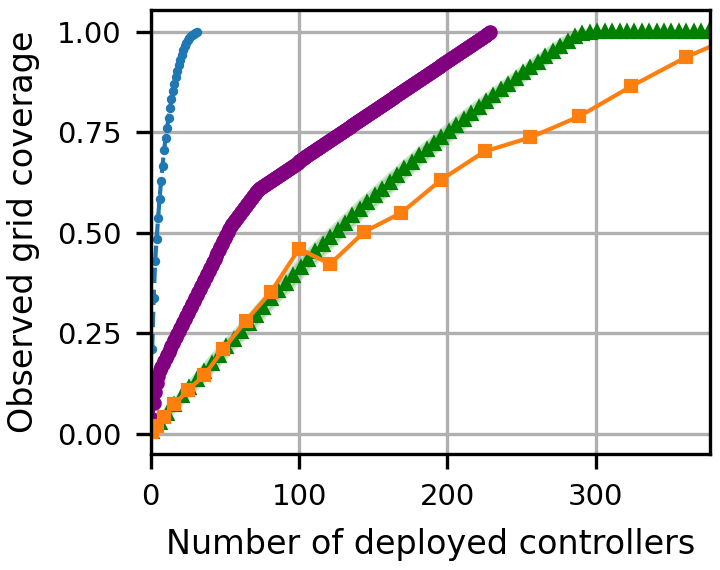}%
        \label{subfig:cartpole_efficiency}}
    \par\smallskip
    {\scriptsize\setlength{\tabcolsep}{2pt}\setlength{\unitlength}{1pt}%
    \begin{tabular}{@{}cl@{\hspace{8pt}}cl@{}}
        \textcolor[HTML]{800080}{\begin{picture}(18,5)
            \put(0,2.5){\line(1,0){18}}\put(9,2.5){\makebox(0,0){$\bullet$}}
        \end{picture}} & Empirical rate ball (greedy) &
        \textcolor[HTML]{1F77B4}{\begin{picture}(18,5)
            \multiput(0,2.5)(6,0){3}{\line(1,0){4}}\put(9,2.5){\makebox(0,0){\tiny$\bullet$}}
        \end{picture}} & Grid-exact enumeration \\
        \textcolor[HTML]{008000}{\begin{picture}(18,5)
            \put(0,2.5){\line(1,0){18}}\put(9,2.5){\makebox(0,0){$\blacktriangle$}}
        \end{picture}} & Random selection &
        \textcolor[HTML]{FF7F0E}{\begin{picture}(18,5)
            \put(0,2.5){\line(1,0){18}}\put(7.5,1){\rule{3pt}{3pt}}
        \end{picture}} & Uniform \revnew{lattice} selection
    \end{tabular}\par}
    \caption{Efficiency comparison of source task selection algorithms. \rev{Observed coverage on fixed finite grids ($\varepsilon=2$ for MSD, $\varepsilon=100$ for CartPole) versus deployed controllers. Random selection: mean and $\pm1$ standard deviation over 50 draws. Uniform \revnew{lattice} selection: $m\times m$ source lattices, plotted at $m^2$.}}
    \label{fig:sample_efficiency_alg}
\end{figure}

\rev{The balls of Lemma~\ref{lemma:gen-bounds} underapproximate the tolerance sets, so they need more sources.} \rev{\revnew{Let} $L^r_{\mathrm{grid}}$ and $m^r_{\mathrm{grid}}$ denote the extrema in Definition~\ref{def:local-bound-rate} restricted to evaluated non-source grid neighbors.} \rev{Greedy selection needs 4 (MSD) and 31 (CartPole) sources with grid-exact tolerance sets, which can extend beyond $r$; 40 and 95 with the largest grid-valid balls of radius at most $r$; 113 and 178 with balls of radius $\min\{r,\varepsilon/L^r_{\mathrm{grid}}\}$; and 139 and 229 with empirical rate balls. The ball restriction and the rate bound account for most of the gap, and estimation with inflation adds 23\% and 29\%, so better rate estimation alone cannot recover the grid-exact libraries.}

\rev{The grid-restricted lower rate $m^r_{\mathrm{grid}}$ is positive at $0.07\%$ of MSD and $54\%$ of CartPole sources.} \rev{\revnew{At} each of the 139 selected sources, the 30 member gradients lie \revnew{in an open half-plane}, so all members improve along a common direction\revnew{, which explains the MSD value. Unlike in Example~\ref{ex:spanning}, positive spanning fails, so $m^r\le0$ at interior sources by Proposition~\ref{prop:sign-mr}(a), and Theorem~\ref{thm:bound-complexity}'s} lower bound does not apply.}

We empirically validate Proposition~\ref{prop:eval-cost} on MSD systems.
\rev{At $N=10^4$, enumeration uses $3.0\times10^5$ evaluations per candidate source versus $4.65\times10^4$ on average for rate estimation, a $6.4\times$ saving. The two costs cross at \revnew{$N\approx1.5\times10^3$}, so enumeration is cheaper on the $20\times20$ CartPole grid.}

\FloatBarrier

\section{Conclusion}\label{sec:conclusion}
We presented a performance-based, directional framework that formalizes and estimates task-space complexity for zero-shot generalization. \rev{Local rates bound \revnew{task-space complexity} by quantities that scale as $\varepsilon^{-d}$. Since they use closed-loop performance rather than model smoothness, the bounds are unaffected by model variation that does not change performance. With certified rate bounds and balls shrunk by the grid covering radius, \revnew{a cover returned by greedy selection} covers the continuum; with estimated rates, the cover is empirical.}

Our experiments on linear and nonlinear systems demonstrate that this framework reduces the number of required controllers \rev{by 53\% on MSD and 21\% on CartPole relative to random selection}. \rev{The empirical rate balls meet the tolerance at all 2780 held-out MSD contexts and at 99.5\% of the 4580 CartPole contexts. Grid-exact tolerance sets yield much smaller libraries; the gap stems mainly from the ball restriction and the rate bound rather than from rate estimation.}

\rev{Future work includes certified rate bounds for structured system families, direction-dependent rates for elongated tolerance sets, \revnew{ensembles whose gradients positively span $\mathbb R^d$}, and joint source--controller selection.}

\bibliographystyle{IEEEtran}
\bibliography{references}

\end{document}